%% file: main.tex
\documentclass[11pt, letterpaper]{article}
\usepackage{graphicx}
\usepackage[
letterpaper,
top=1in,
bottom=1in,
left=1in,
right=1in]{geometry}

\title{Agnostic learning in (almost) optimal time \\ via Gaussian surface area}
\author{
  \begin{minipage}[t]{0.3\textwidth}
    \centering
    Lucas Pesenti\thanks{Supported by the Swiss National Science Foundation (SNSF), grant no. 10004947.} \\ \footnotesize \href{mailto:lucas.pesenti@inf.ethz.ch}{lucas.pesenti@inf.ethz.ch} \\ \small ETH Zurich
  \end{minipage}
  \hfill
  \begin{minipage}[t]{0.3\textwidth}
    \centering
    Lucas Slot\thanks{
    Supported by the project ``Foundations of sum-of-squares algorithms: worst-case and beyond'' with file number VI.Veni.242.234 of the research programme Veni ENW which is (partly) financed by the Dutch Research Council (NWO) under the grant \url{https://doi.org/10.61686/GIEOW65108}.        
    }  
     \\ \footnotesize 
     \href{mailto:l.f.h.slot@uva.nl}{l.f.h.slot@uva.nl} \\ \small University of Amsterdam
  \end{minipage}
  \hfill
  \begin{minipage}[t]{0.3\textwidth}
    \centering
    Manuel Wiedmer\footnotemark[1] \\ \footnotesize \href{mailto:manuel.wiedmer@inf.ethz.ch}{manuel.wiedmer@inf.ethz.ch} \\ \small ETH Zurich
  \end{minipage}
}
\date{March 6, 2026}

\input{preamble}

\begin{document}

\maketitle

\begin{abstract}
The complexity of learning a concept class under Gaussian marginals in the difficult agnostic model is closely related to its $L_1$-approximability by low-degree polynomials. For any concept class with \emph{Gaussian surface area} at most $\Gamma$, Klivans et al. (2008) show that degree $d = O(\Gamma^2 / \e^4)$ suffices to achieve an $\e$-approximation. This leads to the best-known bounds on the complexity of learning a variety of concept classes. In this note, we improve their analysis by showing that degree $d = \tilde O (\Gamma^2 / \e^2)$ is enough. In light of lower bounds due to Diakonikolas et al. (2021), this yields (near) optimal bounds on the complexity of agnostically learning polynomial threshold functions in the statistical query model. Our proof relies on a direct analogue of a construction of Feldman et al. (2020), who considered $L_1$-approximation on the Boolean hypercube.
\end{abstract}

\section{Introduction}
The \emph{agnostic learning} framework of~\cite{kearns1992toward} extends the classical PAC model of~\cite{valiant1984} to learning tasks with noisy input data, and has served as one of the standard models for studying the computational complexity of such tasks. It can be described as follows: Let $\mC$ be a {\em concept class} (i.e., a set of functions $f : \mX\to \{-1,1\}$) and let $\mD$ be a distribution
on $\mX \times \{-1,1\}$. Given labeled examples $(x,y)$ drawn independently from $\mD$, our task is to output a \emph{hypothesis}~$\hat f$ that predicts the labels of future examples drawn from $\mD$. 

Contrary to the PAC model, we do not assume any relation between the examples $x$ and their label $y$. 
Thus, it is generally impossible to find a hypothesis with nontrivial success rate.
Instead, we ask that $\hat f$ predicts labels almost as well as the best concept in $\mC$. 
More formally, we say that an algorithm \emph{agnostically learns} $\mC$ up to excess error $\epsilon >0$ if it produces (with high probability) a hypothesis $\hat f$ with $\mathbb{P}_{(x, y) \sim \mD}(\hat f(x) \neq y) \leq \opt + \epsilon$, where $\opt \coloneqq \min_{f \in \mC} \mathbb{P}_{(x, y) \sim \mD}( f(x) \neq y)$. 
The parameter $\opt$ measures the ``noisiness'' of the problem; if $\opt=0$, we recover the PAC model.

\subsubsection*{The $L_1$-polynomial regression algorithm}
Efficient algorithms for agnostic learning under general distributions are unlikely to exist, already for simple concept classes like halfspaces~\cite{Daniely2016, Tiegel2023}. It is therefore natural to restrict to certain classes of distributions. 
Here, we focus on the setting where $\mX = \R^n$ and the marginal distribution $\mD_\mX = \Gauss$ on the examples is the standard Gaussian.

In their seminal work, \cite{KalaiKlivansMansourServedio:agnosticlearninghalfspaces} give the first efficient algorithm for agnostically learning halfspaces under the Gaussian distribution (among other distributions). 
Their algorithm computes a best \mbox{degree-$d$} polynomial approximation of the labeled examples, as measured in the $L_1$-norm. 
This can be done in time $n^{O(d)}$, e.g., via linear programming.
Assuming that each concept in $\mC$ admits a degree-$d$ approximation with $L_1$-error at most~$\epsilon$, this procedure (combined with a thresholding scheme) yields an agnostic learner for~$\mC$ with excess error~$\epsilon$.\footnote{The choice of $L_1$-norm here is crucial: earlier algorithms that compute a best $L_2$-approximation yield a hypothesis with (total) error $O(\mathrm{opt}) + \epsilon$. This is a weaker guarantee, tolerating only small amounts of noise. For algorithmic results in this regime, see for example~\cite{KlivansLongServedio:Oopt1,AwashtoBalcanLong:Oopt2,Daniely:Oopt3,DiakonikolasKaneStewart:Oopt4,DiakonikolasKontonisTzamosZarifis:Oopt5}.} If $d = d(\epsilon)$ can be chosen independently of the dimension~$n$, this yields a PTAS with runtime $n^{O(d)} \cdot \poly(1/\epsilon)$.

The \emph{$L_1$-polynomial regression algorithm} remains the standard (and de facto only) choice for efficient agnostic learning under distributional assumptions. It achieves the best-known runtime for a variety of concept classes and distributions~\cite{KalaiKlivansMansourServedio:agnosticlearninghalfspaces,KlivansODonnellServedio:Gaussiansurfacearea,DiakonikolasKaneNelson:optimallearningofhalfspaces,DHK+10:PTFoverhypercubefirstresult,BlaisODonnellWimmer:L1regressionforproductdistributions,Kane2011,KaneKlivansMeka:Intersectionsofhalfspacesunderlogconcave,DiakonikolasKaneRen:disjunctionsusingL1regression,ChenPatelServedio:conjunctionsusingL1regression}.
In fact, $L_1$-polynomial regression is known to be (essentially) optimal in the Statistical Query (SQ) model for agnostic learning under the Gaussian distribution~\cite{Diakonikolas2021}.

\subsubsection*{Degree bounds for $L_1$-approximation and Gaussian surface area}
 The results of~\cite{KalaiKlivansMansourServedio:agnosticlearninghalfspaces} and~\cite{Diakonikolas2021} show that the complexity of agnostically learning a concept class under the Gaussian distribution is governed by its approximability by low-degree polynomials in $L_1$: if $d$ is the smallest degree such that each $f \in \mC$ admits an $\epsilon$-approximation in $L_1$ of degree $d$, then the SQ-complexity of agnostically learning~$\mC$ up to error~$\e$ is approximately $n^{\Theta(d)}$.
 
For halfspaces, 
\cite{KalaiKlivansMansourServedio:agnosticlearninghalfspaces} show that degree $d = O(1/\epsilon^4)$ suffices to achieve an $\varepsilon$-approximation in $L_1$. Their proof first establishes the bound on the degree required to $\epsilon$-approximate a halfspace in $L_2$-norm; this can then be converted to an $L_1$-guarantee using Cauchy-Schwarz.
The reason for taking this indirect route is the availability of powerful tools for analyzing approximation in $L_2$ that are unavailable in the $L_1$-setting, most notably Hermite analysis.

To generalize this approach to arbitrary concept classes, \cite{KlivansODonnellServedio:Gaussiansurfacearea} consider the \emph{Gaussian surface area}~(GSA) of a concept $f$, which measures the surface area of the set ${\{x \in \R^n : f(x) = 1\}}$ relative to the Gaussian density function.
Using Hermite analysis, they show that any concept class whose elements have GSA at most~$\Gamma$ can be $\epsilon$-approximated in $L_2$-norm by a polynomial of degree~${d=O(\Gamma^2/\epsilon^4)}$; thus, $d=O(\Gamma^2/\e^4)$ also suffices in~$L_1$. 
Halfspaces, Euclidean balls and certain cones have constant GSA~\cite{KlivansODonnellServedio:Gaussiansurfacearea}; intersections of $k$ halfspaces have~GSA at most $O(\sqrt{\log k})$~\cite{KlivansODonnellServedio:Gaussiansurfacearea};
polynomial threshold functions (PTFs) of degree~$k$ have~GSA at most $O(k)$~\cite{Kane2011}; convex sets have~$\GSA$ at most $O(n^{1/4})$~\cite{Ball1993}.

\subsubsection*{Optimal degree bounds for $L_1$-approximation}
A key appeal of the approach of~\cite{KlivansODonnellServedio:Gaussiansurfacearea} is that it establishes GSA as a universal upper bound on the complexity of agnostic learning.
However, their analysis yields a suboptimal guarantee for approximating (and thus for learning) halfspaces: A direct construction of~\cite{DiakonikolasKaneNelson:optimallearningofhalfspaces} shows that one only needs degree $d=O(1/\epsilon^2)$ to get an $\epsilon$-approximation for halfspaces in the $L_1$-norm (vs. $d=O(1/\epsilon^4)$ obtained in~\cite{KalaiKlivansMansourServedio:agnosticlearninghalfspaces, KlivansODonnellServedio:Gaussiansurfacearea}).

The bound of~\cite{DiakonikolasKaneNelson:optimallearningofhalfspaces} is optimal both in the sense that no better $L_1$-approximation is possible~\cite{Ganzburg2008}, and that the corresponding runtime of the $L_1$-polynomial regression algorithm matches the lower bound of~\cite{Diakonikolas2021} in the SQ-model (up to polynomial factors in $1/\e$). On the other hand, their construction does not (obviously) extend beyond halfspaces (see~\Cref{SEC:comparisontoDKN10}).

This raises the question whether its possible to prove an approximation guarantee that is general (applies to all concepts of bounded GSA) and simultaneously optimal (in terms of its dependence on $\e$). In this work, we (almost) answer this question in the positive: 
We improve the analysis of~\cite{KlivansODonnellServedio:Gaussiansurfacearea} by showing that degree $d = \tilde O(\Gamma^2 / \e^2)$ suffices to achieve an $\e$-approximation of any concept class with GSA at most $\Gamma$ (\Cref{THM:main}). This recovers the (optimal) bound of~\cite{DiakonikolasKaneNelson:optimallearningofhalfspaces} for halfspaces (up to a $\log(1/\e)$-factor). Moreover, it improves the best-known bound for degree-$k$ PTFs from~$O(k^2/\e^4)$ to $\tilde O(k^2/\e^2)$, which (nearly) matches the lower bound of $\Omega(k^2/\e^2)$ from~\cite{Diakonikolas2021}. Finally, it improves the best-known bounds for intersections of halfspaces and convex sets by a factor~$\tilde \Omega(1/\e^2)$.
See~\Cref{TAB:overview} for an overview.

Our technical contributions are minor: We show that a construction of~\cite{Feldman2020} involving the noise operator on the Boolean hypercube can be transported to the Gaussian case (mutatis mutandis). Then, we use results of~\cite{KlivansODonnellServedio:Gaussiansurfacearea} to bound the resulting $L_1$-approximation errors in terms of the Gaussian surface area and conclude. Thus, all ingredients of our proof were (more or less) known in the literature, but (to our knowledge) had not been assembled yet in this context.

\begin{table}[h]
\centering
\renewcommand{\arraystretch}{1.3}
\begin{tabular}{llll}
\hline
\textbf{concept class} & \textbf{LB} & \textbf{UB}  (previous) & \textbf{UB} (Thm.~\ref{THM:main}) \\
\hline
halfspaces 
& $\Omega(1/\varepsilon^2)$~\cite{Ganzburg2008} 
& $O(1/\varepsilon^2)$~\cite{DiakonikolasKaneNelson:optimallearningofhalfspaces} & $\tilde O (1/\epsilon^2)$ \\

degree-$k$ PTFs 
& $\Omega(k^2/\varepsilon^2)$~\cite{Diakonikolas2021} 
& $O(k^2/\varepsilon^4)$~\cite{Kane2011} & $\tilde O (k^2/\epsilon^2)$ \\

intersec. $k$ halfspaces 
& $\tilde{\Omega}(\sqrt{\log k}/\varepsilon)$~\cite{Diakonikolas2021}  
& $O(\log k/\varepsilon^4)$~\cite{KlivansODonnellServedio:Gaussiansurfacearea} & $\tilde O (\log k/\epsilon^2)$ \\
convex sets &
--- & $O(\sqrt{n}/\e^4)~\cite{{KlivansODonnellServedio:Gaussiansurfacearea}}$ & $\tilde O(\sqrt{n}/\e^2)$ \\ 
\hline 
GSA at most $\Gamma$ & --- & $O(\Gamma^2/\e^4)$~\cite{{KlivansODonnellServedio:Gaussiansurfacearea}}& $\tilde O(\Gamma^2/\e^2)$ \\
\hline
\end{tabular}
\caption{Upper bounds (UB) and lower bounds (LB) on the smallest degree needed to achieve an $\e$-approximation in~$L_1(\gauss)$ for concept classes $\mC$. These correspond to upper and lower bounds in~$n^{O(d)}$ and~$n^{\Omega(d)}$ on the (SQ-)complexity of agnostically learning $\mC$ up to error $\e$, respectively.}
\label{TAB:overview}
\end{table}

\subsection{Main results}
To state our main results, we need two (standard) notions related to Boolean functions. First, the \emph{Gaussian noise sensitivity} $\GNS{\rho}(f)$ of ${f : \R^n \to \{\pm 1\}}$ at level $\rho \in [0, 1]$ is the probability that $f(X) \neq f(Y)$ for two $(1-\rho)$-correlated Gaussians random variables $X, Y \sim \gauss$ (see~\Cref{DEF:GNS}). 
Second, writing ${K(f) = \{ x \in \R^n : f(x) = 1\}}$, the \emph{Gaussian surface area} of $f$ is\footnote{In the definition of~\cite{KlivansODonnellServedio:Gaussiansurfacearea}, Gaussian surface area is a property of a subset of $\R^n$. We think of a Boolean function~$f$ as the indicator of the subset $K(f) = \{ x : f(x) = 1 \} \subseteq \R^n$, which allows us to assign a surface area to $f$ itself.}
    \[
        \GSA(f) \coloneqq \lim_{\delta \to 0} \frac{\mathrm{vol}_{\gauss}(K(f)_\delta \setminus K(f))}{\delta},
    \]
where $A_\delta \coloneqq \{ x \in \R^n : \mathrm{dist}(x, A) \leq \delta \}$ is the $\delta$-neighborhood of $A \subseteq \R^n$, and $\mathrm{vol}_{\gauss}(\cdot)$ denotes the Gaussian probability mass (see~\Cref{DEF:GSA}). 

Our main result is the following $L_1$-approximation guarantee for $f$ in terms of its GSA.
\begin{theorem}\label{THM:main}
    Let $f: \R^n \to \{\pm 1\}$ be a (measurable) function. For every $\e > 0$, there exists a polynomial $p$ of degree 
    $
    d 
    \leq O(\log(1/\varepsilon) \cdot \GSA(f)^2/\varepsilon^2)
    $    
    with $
         \expec{x \sim \gauss}{|f(x) - p(x)|}\leq \varepsilon$.
\end{theorem}
As an intermediate result, we get the following bound in terms of the noise sensitivity of $f$. 
\begin{restatable}{proposition}{mainGNS}
\label{PROP:mainGNS}
    Let $f : \R^n \to \{\pm 1\}$ be a (measurable) function. For each $\rho \in [0, 1]$ and $d \in \N$, there is a polynomial $p$ of degree $d$ such that 
    \[
            \expec{x \sim \gauss}{|f(x) - p(x)|} \leq 2\GNS{1 - \rho}(f) + \rho^{d+1}.
    \]
\end{restatable}

Together with \cite[Theorem 5]{KalaiKlivansMansourServedio:agnosticlearninghalfspaces} and bounds on the Gaussian surface area from \cite{Kane2011,KlivansODonnellServedio:Gaussiansurfacearea}, we immediately get the following corollary from~\Cref{THM:main}.
\begin{corollary}
    Let $\mC$ be a concept class of Boolean functions $f: \R^n \to \{\pm 1\}$.
    Let $\Gamma$ be an upper bound on the GSA of every $f \in \mC$.
    Let $\varepsilon > 0$.
    Then, the $L_1$-polynomial regression algorithm agnostically learns $\mC$ up to error $\varepsilon$ in time and sample complexity $n^{\tilde{O}(\Gamma^2/\varepsilon^2)} \cdot \poly(1/\varepsilon)$.
    In particular, the $L_1$ polynomial regression algorithm agnostically learns
    \begin{enumerate}[label={\alph*)},noitemsep]
        \item degree-$k$ PTFs in time $n^{\tilde{O}(k^2/\varepsilon^2)}$;
        \item intersections of $k$ halfspaces in time $n^{\tilde{O}(\log(k)/\varepsilon^2)}$;
        \item convex sets in time $n^{\tilde{O}(\sqrt{n}/\varepsilon^2)}$.
    \end{enumerate}
\end{corollary}

\subsubsection*{Lower bounds on the $L_1$-approximation degree of concepts}
As mentioned, \cite{Diakonikolas2021} show that the SQ-complexity of learning a concept class $\mC$ up to error $\e$ is approximately $n^{\Omega(d)}$, where $d = d(\e)$ is the smallest polynomial degree required to $\e$-approximate any element of $\mC$ in $L_1$. Complementing this result, they establish lower bounds on this degree: For PTFs of degree $k$, they show that $d = \Omega(k^2 / \e^2)$. The case $k=1$ corresponds to halfspaces, thus requiring $d = \Omega(1/\e^2)$. 

They also prove the following general bound in terms of noise sensitivity:

\begin{theorem}[{\cite{Diakonikolas2021}}] \label{THM:GNSLB}  Let $f : \R^n \to \{\pm 1\}$ be a (measurable) function. 
For every polynomial $p$ of degree $d$ we have that 
    \[
    \expec{x \sim \gauss}{|f(x) - p(x)|} \geq \Omega(1 / \log d) \cdot \GNS{(\log d / d)^2}(f)
    \]    
\end{theorem}
As an application, \Cref{THM:GNSLB} shows that $d = \tilde\Omega(\sqrt{\log k} / \e)$ is required to approximate an intersection of $k$ halfspaces. Note that~\Cref{THM:GNSLB} is typically not tight, and does not match the upper bound of~\Cref{PROP:mainGNS}. 
For example, when $f$ is a halfspace, we have $\GNS{\delta}(f) \approx \sqrt{\delta}$. Then, \Cref{THM:GNSLB} gives a degree bound $d = \tilde\Omega(1/\e)$, whereas~\Cref{PROP:mainGNS} shows that $d=\tilde O(1/\e^2)$. The correct degree is $d = \Theta(1/\e^2)$~\cite{Ganzburg2008, DiakonikolasKaneNelson:optimallearningofhalfspaces}.

\subsection*{Outline} The rest of the paper is organized as follows:
In~\Cref{SEC:prelims}, we give some preliminaries on Hermite analysis and the Ornstein-Uhlenbeck (noise) operator; in~\Cref{SEC:proof}, we give a proof of our main result~\Cref{THM:main}; in~\Cref{SEC:comparisons} we compare our construction to those of~\cite{KalaiKlivansMansourServedio:agnosticlearninghalfspaces, DiakonikolasKaneNelson:optimallearningofhalfspaces, Feldman2020}.

\section{Hermite analysis and the Ornstein-Uhlenbeck operator} \label{SEC:prelims}
In this section, we recall some basic facts on Hermite analysis and the Ornstein-Uhlenbeck operator. See~\cite{Ledoux1991, Bakry1994, Janson1997} for general references, and~\cite[Chapter 11]{ODonnell2014} for a nice overview.

\subsection{Hermite analysis} \label{SEC:prelims:hermite}

We write $\gauss = \mathcal{N}(0, I_n)$ for the standard Gaussian distribution on $\R^n$.
Throughout, we assume that all functions $f : \R^n \to \R$ are (Borel-)measurable.
We denote by $L_2(\gauss)$ 
the space of functions ${f: \R^n \to \R}$ with $\|f\|^2_{L_2(\gauss)} \coloneqq \mathbb{E}_{X \sim \gauss}[f(X)^2] < \infty$. It has an inner product given by
${\langle f,g\rangle_{\gauss} \coloneqq \mathbb{E}_{X \sim \gauss}[f(X)g(X)]}$.
Similarly, $L_1(\gauss)$ denotes the space of functions $f: \R^n \to \R$ with $\|f\|_{L_1(\gauss)} \coloneqq \mathbb{E}_{X \sim \gauss} [|f(X)|] < \infty$. We may use Cauchy-Schwarz to bound
\begin{equation} \label{EQ:L1lessL2}
    \|f\|_{L_1(\gauss)} \leq \|f\|_{L_2(\gauss)}\,.
\end{equation}

For $n=1$, the space $L_2(\gaussoned)$ has an orthonormal basis of \emph{Hermite polynomials}~$(H_k)_{k \geq 0} \subseteq \R[x]$, which are the unique polynomials with positive leading coefficient such that $H_k$ has degree $k$ and 
\[
    \langle H_i,~H_j \rangle_{\gaussoned} = \mathbb{E}_{X \sim \gaussoned}[H_i(X) H_j(X)] = \delta_{ij} \quad (\forall i, j \in \N).
\]
This definition corresponds to the (monic) \emph{probabilist's Hermite polynomials} scaled by a factor $1 / \sqrt{k!}$.
For $n > 1$, we have an orthornormal basis of \emph{multivariate Hermite polynomials}, given by
\begin{equation}\label{EQ:hermitemultidimensional}
    H_\alpha(\x) \coloneqq H_{\alpha_1}(x_1) \cdot H_{\alpha_2}(x_2) \cdot \ldots \cdot H_{\alpha_n}(x_n) \quad (\alpha \in \N^n).
\end{equation}
Note that $H_\alpha$ is of degree $|\alpha| \coloneqq \sum_{i=1}^n \alpha_i$.
Any $f \in L_2(\gauss)$ has a (unique) \emph{Hermite expansion}
\[
    f = \sum_{\alpha \in \N^n} \hat f(\alpha) \cdot H_\alpha, \quad \text{with} \quad \hat f(\alpha) \in \R.
\]
We refer to the $\hat f (\alpha)$ as the \emph{Hermite coefficients} of $f$.
By orthonormality of the Hermite polynomials, we have that $\|f\|_{L_2(\gauss)}^2 = \sum_{\alpha} \hat f(\alpha)^2$ (Parseval's identity).
For $d \in \N$, we write $\HP{d} : L_2(\gauss) \to L_2(\gauss)$ for the orthogonal projection onto $\mathrm{span}\{H_\alpha : |\alpha| \leq d\}$. That is,
\[
    \HP{d} f = \sum_{|\alpha| \leq d} \hat f(\alpha) \cdot H_\alpha.
\]
By definition, $\HP{d} f$ is the best $L_2$-approximation of $f$ by a degree-$d$ polynomial. It is also referred to as the degree-$d$ \emph{Hermite expansion} of $f$. Its approximation error can be expressed in terms of the Hermite coefficients of $f$ as follows:
\begin{equation}
\label{EQ:hermitedecay}
    \| f - \HP{d} f\|^2_{L_2(\gauss)} = \Big\| \sum_{|\alpha| > d} \hat f(\alpha) \cdot H_\alpha \Big\|^2_{L_2(\gauss)} = \sum_{|\alpha| > d} \hat f(\alpha)^2. \qedhere
\end{equation}

\subsection{The Ornstein-Uhlenbeck operator}
For $\rho \in [0,1]$,  the \emph{Ornstein-Uhlenbeck operator} $T_\rho : L_2(\gauss) \to L_2(\gauss)$ is defined by:
    \[ 
        T_\rho f(\x) = \mathbb{E}_{Y \sim \gauss}\left[f\left(\rho \x + \sqrt{1-\rho^2} Y\right)\right].
    \]
\noindent
By construction, $T_\rho$ is a linear operator. It is diagonalized by the Hermite polynomials; namely
\begin{equation}
\label{EQ:trhohermitemultdimensional}
    T_\rho H_\alpha = \rho^{|\alpha|} H_\alpha \quad (\forall \alpha \in \N^n).
\end{equation}
For completeness, we include a proof of this fact in~\Cref{LEM:TrhoandHermite}.
As a direct consequence, the Hermite coefficients of~$T_\rho f$ decay exponentially fast for any $f \in L_2(\gauss)$. In light of~\eqref{EQ:hermitedecay}, this means that $T_\rho f$ is well-approximated in $L_2(\gauss)$ by low-degree polynomials:
\begin{lemma} \label{LEM:Trhof-hermiteexpansion}
    Let $f \in L_2(\gauss)$. For any $\rho \in [0, 1]$ and $d \in \N$, we have
    \[
        \| T_\rho f - \Pi_{ \leq d} \big(T_\rho  f\big)\|^2_{L_2(\gauss)} \leq \rho^{2d + 2} \cdot \|f\|_{L_2(\gauss)}^2.
    \]
\end{lemma}
\begin{proof}
Writing $f(x) = \sum_\alpha \hat{f}(\alpha) H_\alpha(x)$, and using~\eqref{EQ:trhohermitemultdimensional}, we find that
    \[
        T_\rho f(x) = \sum_{\alpha \in \N^n} \hat{f}(\alpha) \cdot T_\rho H_\alpha(x) = \sum_{\alpha \in \N^n} \rho^{|\alpha|} \hat{f}(\alpha) \cdot H_\alpha(x).
    \]
    Thus, by~\eqref{EQ:hermitedecay} we get that
    \[
        \| T_\rho f - \Pi_{ \leq d} \big(T_\rho  f\big)\|^2_{L_2(\gauss)} = \sum_{|\alpha| > d} \hat{f}(\alpha)^2 \rho^{2|\alpha|} \leq \rho^{2(d+1)} \cdot \sum_{|\alpha| > d} \hat{f}(\alpha)^2 \leq \rho^{2(d+1)} \cdot \|f\|_{L_2(\gauss)}^2.\qedhere
    \]
\end{proof}

\subsection{Gaussian noise sensitivity and Gaussian surface area}
For $f : \R^n \to \{ \pm 1 \}$, 
the expected difference between $f$ and $T_\rho f$ is equal to twice the \emph{Gaussian noise sensitivity} of $f$. (In fact, this is sometimes used as the definition of noise sensitivity.) For completeness, we include a proof of this fact in~\Cref{SEC:prooffTrhof}.

\begin{definition}[cf. {\cite[Def. 12]{KlivansODonnellServedio:Gaussiansurfacearea}}]\label{DEF:GNS}
     The \emph{Gaussian noise sensitivity of $f : \R^n \to \{ \pm 1 \}$ at $\delta \in [0,1]$} is
    \[
        \GNS{\delta}(f) \coloneqq \mathbb{P}_{(X,Y)}[f(X) \neq f(Y)],
    \]
    where $(X,Y)$ are two $(1-\delta)$-correlated standard Gaussians, that is
    \[
        (X,Y) \sim \mathcal{N}\left(\begin{pmatrix}0_n\\0_n \end{pmatrix}, \begin{pmatrix} I_n & (1-\delta)I_n \\(1-\delta)I_n & I_n \end{pmatrix}\right).
    \]
\end{definition}

\begin{lemma}[see~\Cref{SEC:prooffTrhof}]
\label{LEM:f-Trhof}
Let $f : \R^n \to \{ \pm 1 \}$. For any $\rho \in [0, 1]$, we have
    \[
        \|f - T_\rho f\|_{L_1(\gauss)} = 2 \GNS{1-\rho}(f).
    \]    
\end{lemma}

In turn, the noise sensitivity of a function can be bounded in terms of its \emph{Gaussian surface area}.
This is due to~\cite{KlivansODonnellServedio:Gaussiansurfacearea}, who prove it using a result of~\cite{Ledoux:OrnsteinUhlenbeckandGSA}.

\begin{definition}[cf. {\cite[Def. 1]{KlivansODonnellServedio:Gaussiansurfacearea}}]\label{DEF:GSA}
    Let $K \subseteq \R^n$ be a Borel set. Its Gaussian surface area is 
    \[
        \GSA(K) \coloneqq \lim_{\delta \to 0} \frac{\mathrm{vol}_{\gauss}(K_\delta \setminus K)}{\delta},
    \]
    where $K_\delta \coloneqq \{ x \in \R^n : \mathrm{dist}(x, K) \leq \delta \}$ is the $\delta$-neighborhood of $K$, and $\mathrm{vol}_{\gauss}(\cdot)$ denotes the Gaussian probability mass. 
\end{definition}
For $f : \R^n \to \{\pm 1\}$, we write $K(f) = \{ x \in \R^n : f(x) = 1\}$, and set $\GSA(f) \coloneqq \GSA(K(f))$. We always assume as in~\cite{KlivansODonnellServedio:Gaussiansurfacearea} that  $\mathrm{vol}_{\gauss}\big(\delta K(f)\big) = 0$, where $\delta K(f)$ denotes the boundary of $K(f)$.
\begin{lemma}[{\cite[Corollary 14]{KlivansODonnellServedio:Gaussiansurfacearea}}] \label{LEM:GNSleqGSA} Let $f : \R^n \to \{\pm 1\}$. For each $\rho \in [0,1]$, we have
\[
\GNS{1-\rho}(f) \leq \sqrt{\pi} \sqrt{1-\rho} \cdot \GSA(f).
\]
    
\end{lemma}

\section{Proof of~\texorpdfstring{\Cref{THM:main}}{the main result}}
\label{SEC:proof}

Our proof of~\Cref{THM:main} proceeds in two steps. First, we approximate $f$ by $T_\rho f$, for some $\rho \in [0, 1]$ to be chosen later. Then, we approximate $T_\rho f$ by its low-degree Hermite expansion $\HP{d} \big( T_\rho f \big)$. Using the triangle inequality, the total approximation error is at most
\[
    \|f - \HP{d} \big(T_\rho f\big) \|_{L_1(\gauss)} \leq \|f - T_\rho f \|_{L_1(\gauss)} + \|T_\rho f - \HP{d} \big(T_\rho f\big) \|_{L_1(\gauss)}\,.
\]
The first term on the RHS can be bounded directly in terms of the noise sensitivity of $f$ (\Cref{LEM:f-Trhof}). For the second term, note that, by~\eqref{EQ:L1lessL2} and~\Cref{LEM:Trhof-hermiteexpansion},
\[
    \|T_\rho f - \HP{d} \big(T_\rho f\big) \|_{L_1(\gauss)} \leq \|T_\rho f - \HP{d} \big(T_\rho f\big) \|_{L_2(\gauss)} \leq \rho^{d+1}\,.
\]
Together, this yields:
\mainGNS*\noindent

\begin{proof}[Proof of~\Cref{THM:main}]
    Let $p$ be the degree-$d$ polynomial of \Cref{PROP:mainGNS}. By~\Cref{LEM:GNSleqGSA}, we can further bound
    \begin{equation}
        \|f - p\|_{L_1(\gauss)} \leq 2\sqrt{\pi}\sqrt{1-\rho}\cdot \GSA(f) + \rho^{d+1}\,.\label{eq:propWithGSA}
    \end{equation}
    Given $\e >0$, we want to ensure that both terms on the RHS of~\eqref{eq:propWithGSA} are at most $\e/2$. For the first term, it suffices to set
    \[
        \rho = \max\left(0, 1 - \frac{\varepsilon^2 }{16 \pi \GSA(f)^2}\right).
    \]
    Now, given this choice of $\rho$, we pick $d \in \N$ such that $\rho^{d+1} \leq \e/2$.
    Without loss of generality, assume that $\rho > 0$ (otherwise, $d=0$ suffices).
    We have 
    \begin{align*}
        \rho^{d+1} \leq \frac{\varepsilon}{2} \iff &(d+1)\log(\rho) \leq \log(\varepsilon/2) \\
        \impliedby &(d+1) (\rho-1) \leq \log(\e/2) &&(\text{as } d\geq0 \text{ and } \log(t) \leq t-1 \text{ for all } t > 0) \\  
        \iff &(d+1) \geq \frac{\log(\e/2)}{\rho-1} && (\text{as } \rho-1 < 0)\\
        \iff &d \geq \frac{\log(2/\e)}{1-\rho} -1.
    \end{align*}
    Plugging in the value of $1- \rho$, it therefore suffices to set
    \[
        d = 16\pi\cdot \GSA(f)^2 \cdot \frac{\log(2/\varepsilon)}{\varepsilon^2} - 1.\qedhere
    \]
\end{proof}

\section{Comparison with previous work} \label{SEC:comparisons}
In this section, we compare our proof of~\Cref{THM:main} to three prior works
on agnostic learning that established
polynomial approximation results in $L_1$.
These differ by their choice of low-degree approximation
for $f$:
\begin{enumerate}
    \item the truncated
    Hermite expansion of $f$~\cite{KlivansODonnellServedio:Gaussiansurfacearea};
    \item the truncated
    Hermite expansion of a smoothed version of $f$, obtained by
    convolving with a carefully chosen function~\cite{DiakonikolasKaneNelson:optimallearningofhalfspaces};\footnote{\cite{DiakonikolasKaneNelson:optimallearningofhalfspaces} actually consider the Taylor approximation of this smoothed version of $f$. They are concerned with a stronger notion of approximation than the one considered here (related to proving universality), for which this is important. In our context, we can equivalently consider the truncated Hermite expansion.}
    \item the truncated Fourier expansion of a smoothed version of $f$, obtained by applying the Boolean noise operator~\cite{Feldman2020}.
\end{enumerate}

\subsection{\texorpdfstring{Comparison to ~\cite{KlivansODonnellServedio:Gaussiansurfacearea}}{Comparison to [KOS08]}}

Building on the argument of~\cite{KalaiKlivansMansourServedio:agnosticlearninghalfspaces} for halfspaces, \cite{KlivansODonnellServedio:Gaussiansurfacearea} shows more
generally that the degree-$d$ Hermite
truncation of $f$ (denoted $\Pi_d f$) satisfies
$\|f - \Pi_d f\|_{L_1(\gauss)} \le \varepsilon$ for
$d = O(\GSA(f)^2/\varepsilon^4)$. Crucially, their proof reduces the desired
$L_1$-approximation to an $L_2$-bound:
\begin{align*}
    \| f - \HP{d} f \|_{L_1(\gauss)}
    &\le \| f - \HP{d} f \|_{L_2(\gauss)}\\
    &\le O(1)\cdot \GNS{1/d}(f)^{1/2}
    \tag*{(by~\cite[Prop. 16]{KOS04})}\\
    &\le O(1)\cdot d^{-1/4}\cdot \GSA(f)^{1/2}
    \tag*{(by~\Cref{LEM:GNSleqGSA})}\,.
\end{align*}
The gap between the bound obtained in this argument and the
guarantees of~\Cref{THM:main} can be entirely explained by the
first step: even for origin-centered halfspaces, the $L_2$-error
of the degree-$d$ Hermite truncation is $\Omega(d^{-1/4})$~\cite{KalaiKlivansMansourServedio:agnosticlearninghalfspaces}.
As $\Pi_d f$ is the best
degree-$d$ approximation to $f$ in $L_2$, this shows that the 
strategy of reducing the whole problem to an $L_2$ question is inherently lossy.

However, we are not aware of any example showing that
$\Pi_d f$ fails to match the guarantees of~\Cref{THM:main}
in $L_1$.  On the positive side, we
show that $\Pi_d f$ does achieve these guarantees for origin-centered halfspaces. After reducing to $n=1$, it suffices to show the following:

\begin{proposition} Let $\sign : \R \to \{\pm 1\}$ denote the sign function. We have that
    \label{PROP:sign}
    \[
        \left\|\sign - \Pi_d \sign\right\|_{L_1(\mathcal N_1)}\le  O\left(\frac {\log d} {\sqrt d}\right)\,.
    \]
\end{proposition}
\begin{proof}
See~\Cref{sec:hermite-halfspace}.
\end{proof}
The upshot is that, at least for origin-centered halfspaces, the suboptimal degree bound of~\cite{KalaiKlivansMansourServedio:agnosticlearninghalfspaces, KlivansODonnellServedio:Gaussiansurfacearea} with respect to~\Cref{THM:main} is not due to the choice of approximating polynomial ($\Pi_d f$ vs. $\Pi_d T_\rho f$), but to a loose application of Cauchy-Schwarz.

\subsection{\texorpdfstring{Comparison to~\cite{DiakonikolasKaneNelson:optimallearningofhalfspaces}}{Comparison to [DKN10]}} \label{SEC:comparisontoDKN10}

\cite{DiakonikolasKaneNelson:optimallearningofhalfspaces} 
introduces a different
technique for constructing low-degree approximations, 
showing in particular that degree
$d = O(1/\varepsilon^2)$ suffices to approximate halfspaces 
in $L_1$. In this special case, this improves on the guarantees of~\Cref{THM:main} by a
logarithmic factor. 

In light of our proof, their method can be motivated as follows.  
First, for a
halfspace ${f(x)=\sign(c-\langle w,x\rangle)}$, by projecting onto
$\mathrm{span}(w)$ we may assume without loss of generality that $n=1$.  
We now
identify what properties we would like from a smoothing $g$ of $f$ in order to
implement the strategy from~\Cref{SEC:proof}.  
By Gaussian integration by
parts, the Hermite coefficients bounded in~\Cref{LEM:Trhof-hermiteexpansion} can
be interpreted as (cf. \Cref{LEM:Hermitecoeffsderivatives})
\begin{align}
    \langle g,h_k\rangle = \frac {\E_{x\sim \mathcal N(0,1)} g^{(k)}(x)} {\sqrt{k!}}\,.\label{eq:hermiteDerivative}
\end{align}
Hence, we would like $g$ to approximate $f$ in $L_1$ (to replace \Cref{LEM:f-Trhof}), while keeping the
derivatives of $g$ bounded.  In the one-dimensional case, the authors
of~\cite{DiakonikolasKaneNelson:optimallearningofhalfspaces} show that this is possible in a strong sense: for any fixed $\varepsilon>0$, one can
choose $g$ so that (a)~$\|g-f\|_{L_1(\mathcal N)}\le \varepsilon$, and (b)~${\|g^{(k)}\|_\infty \le O(1/\varepsilon)^{k}}$ for all
$k\ge 1$. In the proof, $g$ is defined from $f$ by convolving with a suitable function whose Fourier transform has
bounded support.
Combined with~\eqref{eq:hermiteDerivative} and following the argument from~\Cref{SEC:proof}, (a) and (b) readily imply that $\|\Pi_d g-f\|_1\le \varepsilon$ for $d=O(1/\varepsilon^2)$, as desired.

The main difficulty with this construction lies in extending it to higher
dimensions.  
For agnostic learning over general concept classes with bounded Gaussian
surface area, repeating the previous argument appears to incur a dimension-dependent factor.  
In
the special case of polynomial threshold functions,
\cite{DiakonikolasKaneNelson:optimallearningofhalfspaces,kanePTF} show that one
can still reduce to the low-dimensional case, at the cost of a (much) worse dependence on
$\varepsilon$.

\subsection{\texorpdfstring{Comparison to~\cite{Feldman2020}}{Comparison to [FKV20]}}\label{SEC:comparisontoFKV20}

The construction in~\cite{Feldman2020} is the Boolean analogue of ours.  Indeed,~\cite[Lemma~3.1]{Feldman2020} shows that for any Boolean function
$f: \{0,1\}^n \to \R$ and $\rho \in (0,1]$, there exists a polynomial $p$ of degree $d = O(\log(1/\varepsilon)/(1-\rho))$ such that
\begin{equation}\label{EQ:TO:L1analogueofFKV}
    \|f-p\|_{L_1(\mU(\{0,1\}^n))} \leq 2 \BNS{1-\rho}(f) + \rho^{d+1} \cdot \|f\|_{L_2(\mU(\{0,1\}^n)}\,,
\end{equation}
where $\mU(\{0,1\}^n)$ is the uniform distribution on $\{0, 1\}^n$ and $\mathrm{BNS}$ is the Boolean noise sensitivity.

Equation~\eqref{EQ:TO:L1analogueofFKV} is the Boolean analogue of
\Cref{PROP:mainGNS}.  
In fact, the proof of~\eqref{EQ:TO:L1analogueofFKV}
in~\cite{Feldman2020} closely parallels our proof of~\Cref{PROP:mainGNS}.  
One
first approximates $f$ by $T_\rho f$, where $T_\rho$ is the noise operator on
the hypercube;
the resulting $L_1$ error is controlled directly by the Boolean noise
sensitivity.  
One then approximates $T_\rho f$ by its degree-$d$ Fourier
truncation, and bounds the $L_1$ error of this truncation by essentially the
same argument as in the Gaussian setting.

\section*{Acknowledgments}
\addcontentsline{toc}{section}{Acknowledgments}
We thank Pjotr Buys for a useful discussion at an early stage of this work.

\bibliographystyle{alpha}
\bibliography{bibliography}

\appendix
\input{appendix}

\end{document}

%% file: preamble.tex
\usepackage[utf8]{inputenc}

\usepackage{amsmath, amsthm, amssymb}
\usepackage{mathtools}

\usepackage[dvipsnames]{xcolor}
\usepackage[T1]{fontenc}
\usepackage[full]{textcomp}
\usepackage[american]{babel}

\usepackage{caption}
\usepackage{booktabs}
\usepackage{siunitx}

\usepackage{newpxtext}
\usepackage{textcomp}
\usepackage[varg,bigdelims]{newpxmath}
\usepackage[scr=rsfso]{mathalfa}
\usepackage{bm}

\usepackage[pagebackref,colorlinks=true,urlcolor=blue,linkcolor=blue,citecolor=OliveGreen]{hyperref}
\usepackage[nameinlink]{cleveref}
\crefname{equation}{}{}

\usepackage{dsfont}

\usepackage{enumitem}

\usepackage{multirow}

\usepackage{algorithm}
\usepackage{algpseudocode}
\algnewcommand{\Input}[1]{\State \textbf{Input:} #1}
\algnewcommand{\Output}[1]{\State \textbf{Output:} #1}

\usepackage{graphicx}

\usepackage{thmtools}
\usepackage{thm-restate}

\newtheorem{theorem}{Theorem}[section]
\newtheorem{corollary}[theorem]{Corollary}
\newtheorem{lemma}[theorem]{Lemma}
\newtheorem{claim}[theorem]{Claim}

\newtheorem{proposition}[theorem]{Proposition}

\newtheorem{fact}[theorem]{Fact}

\theoremstyle{definition}
\newtheorem{definition}[theorem]{Definition}

\renewcommand{\epsilon}{\varepsilon}

\newcommand{\R}{\mathbb{R}}
\newcommand{\N}{\mathbb{N}}

\newcommand{\mC}{\mathcal{C}}

\newcommand{\mX}{\mathcal{X}}
\newcommand{\mD}{\mathcal{D}}
\newcommand{\mU}{\mathcal{U}}

\newcommand{\Gauss}{{\mathcal{N}(0,I_n)}}

\newcommand{\sign}{\mathrm{sign}}

\newcommand{\GSA}{\mathrm{GSA}}
\newcommand{\GNS}[1]{\mathrm{GNS}_{#1}}
\newcommand{\HP}[1]{\Pi_{#1}}

\newcommand{\BNS}[1]{\mathrm{BNS}_{#1}}

\newcommand{\e}{\varepsilon}
\DeclareMathOperator{\opt}{opt}

\newcommand{\poly}{\mathrm{poly}}

\newcommand{\expec}[2]{\mathbb{E}_{#1}\left[ {#2} \right]}

\newcommand{\gauss}{\mathcal{N}_n}
\newcommand{\gaussoned}{\mathcal{N}_1}

\DeclareSymbolFont{bbold}{U}{bbold}{m}{n}
\DeclareSymbolFontAlphabet{\mathbbb}{bbold}

\newcommand{\x}{x}
\newcommand{\y}{y}

\DeclareMathOperator\E{\mathbb E}

%% file: appendix.tex
\section{Technical proofs}\label{APP:technicalproofs}
In this appendix, we give some technical proofs omitted from the main text.

\subsection{Eigenvalues of the Ornstein-Uhlenbeck operator}
We show that the Ornstein-Uhlenbeck operator is diagonalized by the Hermite polynomials (equation~\eqref{EQ:trhohermitemultdimensional}). This is a standard result which we include for completeness, see, e.g.,~\cite{Bakry1994}.
We need the following fact about the generating function for the Hermite polynomials.

\begin{fact}\label{FACT:generatingfunctionforhermite}
    We have that, for every $x,t \in \R$,
    \[
        \exp\left(xt-\frac{1}{2}t^2\right) = \sum_{n=0}^\infty H_k(x) \frac{t^k}{\sqrt{k!}}.
    \]
\end{fact}
\noindent
By definition of $T_\rho$, for any $f$ of the form $f(x) = f_1(x_1) \cdot f_2(x_2) \cdot \ldots \cdot f_n(x_n)$, we have
\begin{equation*}
    T_\rho f (\x) = T_\rho f_1(x_1) \cdot T_\rho f_2(x_2) \cdot \ldots \cdot T_\rho f_n(x_n).
\end{equation*}
Together with \eqref{EQ:hermitemultidimensional}, it thus suffices to prove the following one-dimensional case to conclude \eqref{EQ:trhohermitemultdimensional}.

\begin{lemma}\label{LEM:TrhoandHermite}
    For any $n \in \N$, we have that
    \[
        T_\rho H_n = \rho^n H_n.
    \]
\end{lemma}
\noindent
The proof presented below follows \cite[Section 0.I]{Bakry1994}.
\begin{proof}
    We want to use \Cref{FACT:generatingfunctionforhermite}. Define $g(x) = \exp(xt - \frac{1}{2}t^2)$. Then, by definition of $T_\rho$ and $g$, we get that
    \begin{align*}
        T_\rho g(x) &= \mathbb{E}_{Y \sim \gauss}\left[g\left(\rho x + \sqrt{1-\rho^2} Y\right)\right]\\
         &= \int_{y \in \R} \exp\left(\rho x t + \sqrt{1-\rho^2}y t -\frac{1}{2}t^2\right) \frac{1}{\sqrt{2\pi}} \exp\left(-\frac{1}{2}y^2\right) \mathrm{d}y.
    \end{align*}
    Rearranging the terms and completing the square, we get that
    \begin{align*}
        T_\rho g(x) &= \exp\left(\rho x t - \frac{1}{2}t^2\right) \int_{y \in \R} \frac{1}{\sqrt{2\pi}} \exp\left(-\frac{1}{2}\left(y -\sqrt{1-\rho^2}t\right)^2\right) \exp\left(\frac{1}{2}\left(1-\rho^2\right)t^2\right)\mathrm{d}y\\
        &= \exp\left(\rho x t - \frac{1}{2} \rho^2 t^2\right) \int_{y \in \R} \frac{1}{\sqrt{2\pi}} \exp\left(-\frac{1}{2}\left(y -\sqrt{1-\rho^2}t\right)^2\right) \mathrm{d}y.
    \end{align*}
    Substituting now $y' = y - \sqrt{1-\rho^2}t$, we get that
    \[
        \int_{y \in \R} \frac{1}{\sqrt{2\pi}} \exp\left(-\frac{1}{2}\left(y -\sqrt{1-\rho^2}t\right)^2\right) \mathrm{d}y = \int_{y \in \R} \frac{1}{\sqrt{2\pi}} \exp\left(-\frac{1}{2}\left(y'\right)^2\right) \mathrm{d}y' = 1
    \]
    and thus we get that
    \[
        T_\rho g(x) = \exp\left(\rho x t - \frac{1}{2} \rho^2 t^2\right).
    \]
    We now use \Cref{FACT:generatingfunctionforhermite} twice. First, we get that (using \Cref{FACT:generatingfunctionforhermite} for $t' = \rho t$)
    \[
        T_\rho g(x) = \exp\left(\rho x t - \frac{1}{2} \rho^2 t^2\right) = \sum_{n=0}^\infty H_n(x) \rho^n \frac{t^n}{\sqrt{n!}}.
    \]
    On the other hand, first using \Cref{FACT:generatingfunctionforhermite} and then linearity of $T_\rho$, we get
    \[
        T_\rho g(x) = T_\rho\left(\sum_{n=0}^\infty H_n(x) \frac{t^n}{\sqrt{n!}}\right) = \sum_{n=0}^\infty T_\rho H_n(x) \frac{t^n}{\sqrt{n!}}.
    \]
    Thus, we can conclude that
    \[
        \sum_{n=0}^\infty H_n(x) \rho^n \frac{t^n}{\sqrt{n!}} = \sum_{n=0}^\infty T_\rho H_n(x) \frac{t^n}{\sqrt{n!}}
    \]
    and since this holds for all $t \in \R$, we need to have
    \[
        T_\rho H_n = \rho^n H_n.\qedhere
    \]
\end{proof}

\subsection{Ornstein-Uhlenbeck operator and noise sensitivity} \label{SEC:prooffTrhof}
Here, we give a proof of~\Cref{LEM:f-Trhof}, showing that $\|f - T_\rho f\|_{L_1(\gauss)} = 2\GNS{1 - \rho}(f)$ for $f : \R^n \to \{\pm 1\}$.
\begin{proof}[Proof of \Cref{LEM:f-Trhof}]
    We have that
    \begin{align*}
        \|f - T_\rho f\|_{L_1(\gauss)} &= \mathbb{E}_{X \sim \gauss}\left[|f(X)- T_\rho f(X)|\right]\\
        &= \mathbb{E}_{X \sim \gauss}\left[\left|f(X)- \mathbb{E}_{Y \sim \gauss}\left[f\left(\rho X + \sqrt{1-\rho^2}Y\right)\right] \right|\right]\\
        &= \mathbb{E}_{X \sim \gauss}\left[\left|\mathbb{E}_{Y \sim \gauss}\left[f(X)- f\left(\rho X + \sqrt{1-\rho^2}Y\right)\right] \right|\right]\\
        &= \mathbb{E}_{X,Y \sim \gauss}\left[\left|f(X)- f\left(\rho X + \sqrt{1-\rho^2}Y\right)\right|\right].
    \end{align*}
    In the last step, we used the triangle inequality. For a function $f: \R^n \to \{\pm 1\}$ this is in fact an equality because for fixed $x$, the term $f(\x)- f(\rho \x + \sqrt{1-\rho^2}\y)$ is either always non-negative (if $f(x) = 1$) or always non-positive (if $f(x) = -1$).
    
    Now, if $X, Y \sim \mathcal{N}(0,I_n)$, then $X$ and $\rho X + \sqrt{1-\rho^2}Y$ are $\rho$-correlated Gaussians. Therefore, as
    \[
        \left|f(\x)- f\left(\rho \x + \sqrt{1-\rho^2}\y\right)\right| = 2 \cdot \mathbbb{1}\left\{f(\x) \neq f\left(\rho \x + \sqrt{1-\rho^2}\y\right)\right\}
    \]
    the above expectation is equal to $2\GNS{1-\rho}(f)$.
\end{proof}

\subsection{Hermite coefficients of a smooth function}
We show how to compute the Hermite coefficients of a smooth function in terms of its derivatives.
We need the following fact about the Hermite polynomials. Let $\varphi$ denote the density function of the standard Gaussian on $\R$.
\begin{fact}\label{FACT:Hermiteandgausspdf}
    For $x \in \R$ and $k \in \N_0$ we have that
    \[
        H_k(x)\varphi(x) = \frac{(-1)^k}{\sqrt{k!}} \cdot \frac{d^k\varphi}{dx^k}(x).
    \]
\end{fact}
\noindent
With this fact, we can now prove the following lemma.
\begin{lemma}\label{LEM:Hermitecoeffsderivatives}
    Let $g: \R^n \to \R$ be a smooth function with $\left\|\frac{\partial g}{\partial\alpha}\right\|_\infty < \infty$ for all multi-indices $\alpha$.
    Let $\beta$ be a multi-index.
    Then we have that
    \[
        \hat{g}(\beta) \coloneqq \langle g, H_\beta\rangle_{\gauss} = \frac{\mathbb{E}_{X \sim \gauss}\left[\frac{\partial g}{\partial\beta}(X)\right]}{\sqrt{\beta!}}.    \]
\end{lemma}
\begin{proof}
    In order to prove this lemma, we want to use Gaussian integration by parts on every coordinate.
    For this, consider first the case $n = 1$.
    In this case we have, by \Cref{FACT:Hermiteandgausspdf},
    \[
        \langle g, H_k \rangle_{\gaussoned} = \frac{(-1)^k}{\sqrt{k!}}\int_{x \in \R} g(x) \cdot \frac{d^k\varphi}{dx^k}(x) \mathrm{d}x.
    \]
    Using integration by parts $k$ times, we get
    \[
        \langle g, H_k \rangle_{\gaussoned} = \frac{1}{\sqrt{k!}}\int_{x \in \R} \frac{d^k g}{dx^k}(x) \cdot \varphi(x) \mathrm{d}x = \frac{1}{\sqrt{k!}}\mathbb{E}_{X \sim \gauss}\left[ \frac{d^k g}{dx^k}(X)\right].
    \]
    Note that the boundary conditions are always $0$ because we assume $\left\|\frac{d^\ell g}{dx^\ell}\right\|_\infty < \infty$.
    This shows the one-dimensional case.
    For the multivariate case note that we can iteratively consider each coordinate to get
    \begin{align*}
        \hat{g}(\beta) &\coloneqq \langle g, H_\beta\rangle_{\gauss}\\
        &= \mathbb{E}_{X \sim \gauss}[g(X) \cdot H_{\beta_1}(X_1) \cdot \ldots \cdot H_{\beta_n}(X_n)]\\
        &= \mathbb{E}_{X_2,\ldots,X_{n} \sim \gaussoned}[\mathbb{E}_{X_1 \sim \gaussoned}[g(X) \cdot H_{\beta_1}(X_1) \cdot \ldots \cdot H_{\beta_n}(X_n) \mid X_2, \ldots, X_n]]\\
        &= \mathbb{E}_{X_2,\ldots,X_{n} \sim \gaussoned}[\mathbb{E}_{X_1 \sim \gaussoned}[g(X) \cdot H_{\beta_1}(X_1) \mid X_2, \ldots, X_n]  \cdot H_{\beta_2}(X_2) \ldots \cdot H_{\beta_n}(X_n)].
    \end{align*}
    Now, we can use the one-dimensional case to get 
    \[
        \mathbb{E}_{X_1 \sim \gaussoned}[g(X) H_{\beta_1}(X_1) \mid X_2, \ldots, X_n] = \frac{1}{\sqrt{\beta_1!}} \mathbb{E}_{X_1 \sim \gaussoned}\left[\frac{\partial^{\beta_1} g}{\partial x_1^{\beta_1}}(X) \mid X_2, \ldots, X_n\right]
    \]
    and thus
    \[
        \hat{g}(\beta) = \frac{1}{\sqrt{\beta_1!}} \mathbb{E}_{X \sim \gauss}\left[\frac{\partial^{\beta_1} g}{\partial x_1^{\beta_1}}(X) \cdot H_{\beta_2}(X_2) \ldots \cdot H_{\beta_n}(X_n)\right].
    \]
    We can iteratively use the same argument for the other coordinates to get the result.
\end{proof}

\subsection{Hermite truncation for halfspaces}
\label{sec:hermite-halfspace}

\newcommand\diff{\,\mathrm d}

The goal of this subsection is to prove~\Cref{PROP:sign}.
All $O$ notations are with
respect to $d\to\infty$.
We denote by $\varphi$ the density
of the standard Gaussian 
distribution on $\R$. 

Note that $\sign\in L^2(\mathcal N_1)$, so the Hermite expansion of $\sign$ is well-defined. Moreover, since $\sign$ is odd, all even-degree Hermite
coefficients of $\sign$ are 0, so it suffices to consider
the case of odd~$d$. By the same symmetry,
we can simplify
\begin{equation}
    \|\sign - \Pi_d \sign\|_{L_1(\mathcal N_1)} =     \int_{-\infty}^\infty |\sign(x) - \Pi_d \sign(x)|\diff \varphi(x) = 2 \int_{0}^\infty |1 - \Pi_d \sign(x)| \diff \varphi(x)\,.\label{eq:target-hermite}
\end{equation}
To bound the integral on the
right-hand side, we first express the low-degree
Hermite projection of $\sign$
in terms of Hermite polynomials:

\begin{lemma}
    \label{lem:expansion-sign}
    If $d$ is odd, then for every $x\ge 0$,
    \[
        \Pi_d \sign(x) = \sqrt{\frac {2d} \pi} \,H_{d-1}(0) \int_0^x \frac {H_d(t)} t \diff t\,.
    \]
\end{lemma}

The integral on the right-hand side is well-defined for any fixed odd~$d$, because, as $t\to 0$,
we have 
$H_d(t) \sim C(d) t$ for some
constant $C(d)>0$.

\begin{proof}
    For $k$ odd, the 
    $k$-th Hermite coefficient
    of $\sign$ is
    \[
        \langle \sign, H_k\rangle_{\mathcal N_1} = \int_{-\infty}^\infty \sign(y) H_k(y) \diff \varphi(y) = 2\int_0^\infty H_k(y) \diff \varphi(y)\,.
    \]
    Next, we use the identity $H_k \varphi = -\frac 1 {\sqrt k} (H_{k-1} \varphi)'$ (which follows from the recurrence
    relation that the orthonormalized Hermite polynomials satisfy) to obtain
    \[
        \langle \sign, H_k\rangle_{\mathcal N_1} = \frac2 {\sqrt k} \left[-H_{k-1}(y) \varphi(y)\right]^\infty_0 = \sqrt{\frac 2 {k\pi}} H_{k-1}(0)\,.
    \]
    Therefore, for any $x\in \R$,
    \[
        \Pi_d \sign(x) = \sum_{k=1}^d \langle \sign, H_k\rangle_{\mathcal N_1} H_k(x) = \sqrt{\frac 2 \pi} \sum_{k=1}^d \frac {H_k(x)  H_{k-1}(0)} {\sqrt k}\,.
    \]
    Differentiating both
    sides with respect to $x$ and applying
    the identity $H_k' = \sqrt k H_{k-1}$,
    \[
        (\Pi_d \sign)'(x) = \sqrt{\frac 2 \pi} \sum_{k=1}^d H_{k-1}(x) H_{k-1}(0)\,.
    \]
    Now, we apply the Christoffel-Darboux formula~\cite[Theorem 3.2.2]{Szego:orthogonalpolynomials}:
    \[
        \sum_{k=0}^{d-1} H_k(x) H_k(0) = \sqrt{d} \,\frac {H_d(x) H_{d-1}(0) - H_{d-1}(x) H_d(0)} {x}\,.
    \]
    Since $d$ is odd, $H_d(0)=0$,
    so we are left with:
    \[
        (\Pi_d \sign)'(x) = \sqrt{\frac {2d} \pi} H_{d-1}(0) \frac {H_d(x)} x\,,
    \]
    and
    integrating this identity yields the claim.
\end{proof}

\subsubsection{Asymptotics of Hermite polynomials}
\label{sec:hermite-asymp}

We will use a pointwise asymptotic expansion of
Hermite polynomials $H_d(x)$ 
as $d\to\infty$.

The expansion
holds uniformly in the region $|x|\ll \sqrt d$ and is
known as the {\em Plancherel-Rotach asymptotics}~\cite{Szego:orthogonalpolynomials}. It takes
the form
\begin{equation}
    H_d(x) = \exp\left(\frac{x^2}{4}\right) \cdot \left(\frac{2}{\pi d}\right)^{\frac{1}{4}} \cdot \left(\sin\left[\frac {1-d} 2 \pi + \sqrt{d}x\right] + r_d(x)\right)\,,\label{eq:def-rem}
\end{equation}
which defines the remainder term $r_d(x)$ that we will bound next.
Note that the argument of $\sin$ can be simplified
depending on the parity of $d$:
\[
    \sin\left[\frac {1-d} 2 \pi + \sqrt{d}x\right] =
        \begin{cases}
        (-1)^{\frac{d-1}{2}}\sin(\sqrt{d}x) & \text{if $d$ is odd}\\
        (-1)^{\frac d2}\cos(\sqrt{d}x) & \text{if $d$ is even}
        \end{cases}
\]

\begin{lemma}
    \label{claim:plancherel-rotach}
    For all $|x| \leq \sqrt d$, we have
    \[
        |r_d(x)|\le  O\left(\max\left\{\frac{|x|^3}{\sqrt{d}},\frac{|x|}{\sqrt{d}},\frac{x^2}{d},\frac{1}{d}\right\}\right)\,,
    \]
    where the $O$ is uniform over all $|x| \leq \sqrt d$.
\end{lemma}

\begin{proof}
    We use the following formula~\cite[Theorem 8.22.9]{Szego:orthogonalpolynomials}:
    \begin{align*}
        H_d(x) &= \exp\left(\frac{x^2}{4}\right) \cdot \left(\frac{2}{\pi d}\right)^{\frac{1}{4}} \cdot \sin\left[\arccos\left(\frac{x}{\sqrt{4d+2}}\right)\right]^{-\frac{1}{2}}\cdot\\
       & \mathrel{\phantom{=}}\left(\sin\left[\left(\frac{d}{2}+\frac{1}{4}\right)\left(2\frac{x}{\sqrt{4d+2}}\sqrt{1-\frac{x^2}{4d+2}}-2\arccos\left(\frac{x}{\sqrt{4d+2}}\right)\right)+\frac{3\pi}{4}\right] + O\left(\frac{1}{d}\right)\right).
    \end{align*}
    This holds whenever $|x| \leq \frac{1}{2}\sqrt{4d+2}$ and the $O$ is uniform over all such $x$.
    We use Taylor expansion to get that $\sin[\arccos(x/\sqrt{4d+2})] = 1 + O(x^2/d)$, where the $O$ is again uniform over all $|x| \leq \frac{1}{2}\sqrt{4d+2}$.
    Defining $g(y) = 2y\sqrt{1-y^2}-2\arccos\left(y\right)$, we can thus write
    \[
        H_d(x) = \exp\left(\frac{x^2}{4}\right) \cdot \left(\frac{2}{\pi d}\right)^{\frac{1}{4}} \cdot \left(\sin\left[\left(\frac{d}{2}+\frac{1}{4}\right)\cdot g\left(\frac{x}{\sqrt{4d+2}}\right)+\frac{3\pi}{4}\right] + O\left(\max\left\{\frac{x^2}{d},\frac{1}{d}\right\}\right)\right).
    \]
    Using the Taylor expansion of $g$ around $0$, we get $g(y) = -\pi + 4y + \frac{1}{2}g''(\xi)y^2$ for some $\xi$ between~$0$ and $y$.
    We can bound the derivative $g''(\xi)$ by $O(|y|)$, for some uniform constant for all $|y| \leq \frac{1}{2}$.
    Putting this all together and using that $\sin$ is 1-Lipschitz, we get
    \begin{align*}
        H_d(x) 
        &= \exp\left(\frac{x^2}{4}\right) \cdot \left(\frac{2}{\pi d}\right)^{\frac{1}{4}} \cdot \left(\sin\left[\left(\frac{d}{2}+\frac{1}{4}\right)\cdot \left(-\pi + 4\frac{x}{\sqrt{4d+2}}\right)+\frac{3\pi}{4}\right] + O\left(\max\left\{\frac{|x|^3}{\sqrt{d}},\frac{x^2}{d},\frac{1}{d}\right\}\right)\right)\\
        &= \exp\left(\frac{x^2}{4}\right) \cdot \left(\frac{2}{\pi d}\right)^{\frac{1}{4}} \cdot \left(\sin\left[\frac{1-d}{2}\pi + \sqrt{d+\frac{1}{2}}x\right] + O\left(\max\left\{\frac{|x|^3}{\sqrt{d}},\frac{x^2}{d},\frac{1}{d}\right\}\right)\right)\,.
    \end{align*}
    Next, we note that $\sqrt{d+\frac{1}{2}} = \sqrt{d} + O(1/\sqrt{d})$ and hence, using Lipschitzness of $\sin$ again, we get
    \[
        H_d(x) = \exp\left(\frac{x^2}{4}\right) \cdot \left(\frac{2}{\pi d}\right)^{\frac{1}{4}} \cdot \left(\sin\left[\frac{1-d}{2}\pi + \sqrt{d}x\right] + O\left(\max\left\{\frac{|x|^3}{\sqrt{d}},\frac{|x|}{\sqrt{d}},\frac{x^2}{d},\frac{1}{d}\right\}\right)\right),
    \]
    as claimed.
\end{proof}

\begin{corollary}
    \label{claim:equiv-zero}
    If $d$ is even, then
    \[
        H_d(0) = (-1)^{\frac d 2}\left(\frac 2 {\pi}\right)^{\frac 14} d^{-\frac 14} + O(d^{-\frac 54})\,.
    \]
\end{corollary}

\begin{corollary}\label{cor:uniform-small-hermite}
    We have
    \[
        \sup_{t\in [0,d^{1/6}]} |H_d(t)|e^{-t^2/4}\le O(d^{-\frac 14})\,.
    \]
\end{corollary} 

\begin{corollary}\label{cor:small-r}
    If $d$ is odd, then
    \[
        \sup_{t\in (0,d^{-1/2}]} \frac {|r_d(t)|} t \le O\left(\frac 1 {\sqrt d}\right)\,.        
    \]
\end{corollary}

\begin{proof}[Proof of~\Cref{cor:small-r} from~\Cref{claim:plancherel-rotach}]
    We apply Taylor's inequality (for $d$ odd, we have
    $r_d(0)=0)$:
    \begin{equation}\label{eq:taylor-rd}
        \sup_{t\in (0,d^{-1/2}]} \frac {|r_d(t)|} t\le \sup_{t\in (0,d^{-1/2}]} |r_d'(t)|\,.
    \end{equation}
    Next, we compute $r'_d(t)$. Differentiating~\eqref{eq:def-rem} term by term, we get
    \[
        H_d'(t) = \left(\frac 2 {\pi d}\right)^{1/4} e^{t^2/4}\left( \frac t 2 ((-1)^{\frac{d-1} 2}\sin(t\sqrt d) + r_d(t)) + \sqrt d (-1)^{\frac{d-1} 2}\cos(t\sqrt d) + r_d'(t)\right)\,.
    \]
    Using the identity $H_d'(t) = \sqrt d H_{d-1}(t)$, we obtain, after rearranging:
    \[
        r_d'(t) = \sqrt d \left[\left(\frac {\pi d} 2\right)^{1/4} e^{-t^2/4} H_{d-1}(t)  - (-1)^{\frac{d-1} 2} \cos(t\sqrt d)\right] -  \frac t 2 \left[ (-1)^{\frac{d-1} 2} \sin(t \sqrt d) + r_d(t)\right]\,. 
    \]
    We show separately that the first and the second term are small in absolute value. 
    The second term is $O(d^{-1/2})$ whenever
    $t\le d^{-1/2}$. For the first term, we apply~\Cref{claim:plancherel-rotach} to $H_{d-1}(t)$ (in
    the regime $t\le d^{-1/2}$, the error term $1/d$ dominates):
    \[
        \sup_{t\in [0, d^{-1/2}]} \left|\left(\frac {\pi d} 2\right)^{1/4} e^{-t^2/4} H_{d-1}(t) - (-1)^{\frac {d-1} 2} \cos(t\sqrt d)\right|\le O\left(\frac 1 d\right)\,.
    \]
    The proof follows from 
    multiplying by $\sqrt d$
    and substituting the bound
    in~\eqref{eq:taylor-rd}.
\end{proof}

\subsubsection{\texorpdfstring{Proof of \Cref{PROP:sign}}{Proof}}

Applying the Hermite asymptotic
expansion of $H_{d-1}(0)$ to our target~\eqref{eq:target-hermite},
\begin{align*}
    \Pi_d \sign(x) &= \sqrt{\frac {2d} \pi} H_{d-1}(0) \int_0^x \frac {H_d(t)} t \diff t\tag*{(by \Cref{lem:expansion-sign})}\\
    &= \left((-1)^{\frac {d-1} 2} \left(\frac 2 {\pi}\right)^{\frac 34} d^{\frac 14} + O(d^{-\frac 34} )\right) \int_0^x \frac {H_d(t)} t \diff t   \tag*{(by \Cref{claim:equiv-zero})}\,.
\end{align*}
The main technical
part of the argument lies in bounding the integral
of $t\mapsto H_d(t)/t$.

\begin{lemma}[Small $t$]\label{lem:small-hermite}
    If $d$ is odd, then for all $\tau=\tau(d)\le d^{1/6}$,
    \[
        \int_0^{\tau} \frac {|H_d(t)|} t \diff t \le O(d^{\frac14})\cdot \tau e^{\tau^2/4}\,.
    \]
\end{lemma}

\begin{proof}
    Since $H_d(0)=0$ (when $d$ is odd),  Taylor's inequality yields:
    \[
        \left|\int_0^{\tau} \frac {H_d(t)} t \diff t\right|\le \tau\cdot \sup_{t\in [0,\tau]} |H_d'(t)|\,.
    \]
    The result follows from the identity
    $H_d'(t) = \sqrt d \,H_{d-1}(t)$ and~\Cref{cor:uniform-small-hermite}.
\end{proof}

\begin{lemma}[Large $t$]\label{lem:large-hermite}
    For all $\tau\ge 1$,
    \[
        \int_0^\infty \int_{\tau}^{x} \frac {|H_d(t)|} t \diff t \diff \varphi(x)\le e^{-\tau^2/4}\,.
    \]
\end{lemma}

\begin{proof}
    Using the notation $\bar\Phi(t) = \Pr_{x\sim \mathcal N(0,1)}(x>t)$, we have
    \begin{align*}
        \int_0^\infty \int_{\tau}^{x} \frac {|H_d(t)|} t \diff t \diff \varphi(x) &=  \int_{\tau}^\infty \frac{|H_d(t)|} t \,\bar \Phi(t) \diff t\tag*{(by Fubini's theorem)}\\
        &\le \left(\int_{\tau}^\infty H_d(t)^2 \varphi(t) \diff  t\right)^{\frac12} \left(\int_{\tau}^\infty  \frac {\bar \Phi(t)^2} {t^2 \varphi(t)} \diff t\right)^{\frac12}\tag*{(by Cauchy-Schwarz)}\\
        &\le \left(\int_{\tau}^\infty  \varphi(t)\diff t\right)^{\frac12}\tag*{(orthonormality, $\bar\Phi(t)\le \varphi(t)$, and $\tau\ge 1$)}\\
        &\le e^{-\tau^2/4}\tag*{($\bar\Phi(\tau)\le e^{-\tau^2/2}$)}\,.
    \end{align*}
    
\end{proof}

\begin{proof}[Proof of~\Cref{PROP:sign}]
    We start by expanding $H_{d-1}(0)$ with~\Cref{claim:equiv-zero}:
    \begin{align*}
        &\int_0^\infty |1-\Pi_d\sign(x)|\diff \varphi(x) \\
        &= \int_0^\infty \left|1 - (-1)^{\frac{d-1} 2} \left(\frac 2 \pi\right)^{\frac 34} d^{\frac 14}\int_0^x \frac {H_d(t)} t \diff t\right| \diff \varphi(x) + O(d^{-\frac 34}) \cdot  \int_0^\infty \int_0^x \frac {|H_d(t)|} t \diff t \diff \varphi(x)\,.
    \end{align*}
    We bound the second term using~\Cref{lem:small-hermite} when $t\le d^{1/6}$ and \Cref{lem:large-hermite} when $t>d^{1/6}$:
    \begin{align*}
        d^{-\frac 34} \int_0^\infty \int_0^x \frac {|H_d(t)|} t \diff t \diff \varphi(x)&\le d^{-\frac 34} \left(\int_0^\infty O(d^{1/4}) \cdot xe^{x^2/4}  \diff \varphi(x)+ e^{-d^{1/3}/4}\right)\le O(d^{-\frac 12})\,.
    \end{align*}
    This concludes the analysis of the second term.
    Next, we split the inner integral of the first term into $t\in [0,\tau]$ and $t\in [\tau,\infty)$ for $\tau = \sqrt{100\log d}$.
    We start with the latter, which is negligible
    by~\Cref{lem:large-hermite}:
    \[
        d^{\frac 14} \int_0^\infty \int_{\tau}^{x} \frac{|H_d(t)|} t \diff t \diff\varphi(x)\le d^{\frac 14} e^{-\tau^2/4} = O(d^{-\frac 12})\,.
    \]
    We switch to $t\in [0,\tau]$, which gives the
    main contribution.
    We decompose as follows:
    \begin{align*}
        \int_0^\infty \left|1 - (-1)^{\frac{d-1} 2} \left(\frac 2 \pi\right)^{\frac 34} d^{\frac 14}\int_0^{\min(x,\tau)} \frac {H_d(t)} t \diff t\right| \diff \varphi(x)
        &\le \int_0^\infty \left|1 - \frac 2 \pi \int_0^{\min(x,\tau)} \frac {\sin(t\sqrt d)} t \diff t\right| \diff \varphi(x)\\
        &+\frac 2 \pi \int_0^\infty \left|\int_0^{\min(x,\tau)} \sin(t\sqrt d) \cdot \frac {e^{t^2/4} - 1} t \right|\diff \varphi(x)\\
        &+\frac 2 \pi \int_0^\infty \left|\int_0^{\min(x,\tau)} \frac {e^{t^2/4} r_d(t)} t \right|\diff \varphi(x)\,.
    \end{align*}
    where we recall that $r_d(t)$ is defined in~\eqref{eq:def-rem}.
    We will bound the three terms separately:
        \begin{align}
            \int_0^\infty \left|1-\frac 2 \pi \int_{0}^{\min(\tau,x)} \frac{\sin(t\sqrt d)} t \diff t\right|\diff \varphi(x) &\le O\left(\frac {\log d} {\sqrt d}\right)\label{lem:inter-main}\\
        \int_0^\infty \left|\int_{0}^{\min(x,\tau)} \frac{e^{t^2/4}-1} t \sin(t\sqrt d) \diff t\right|\diff \varphi(x) &\le O\left(\frac {1} {\sqrt d}\right)\label{lem:inter-rem1}\\
        \int_0^\infty \int_{0}^{\min(x,\tau)} {e^{t^2/4}} \frac {|r_d(t)|} t \diff t \diff \varphi(x) &\le O\left(\frac{1} {\sqrt d}\right)\,.\label{lem:inter-rem2}        
        \end{align}
    Assuming~\eqref{lem:inter-main},~\eqref{lem:inter-rem1}, and~\eqref{lem:inter-rem2}, the proof
    of~\Cref{PROP:sign} follows from the triangle inequality.
\end{proof}

\begin{proof}[Proof of~\eqref{lem:inter-main}]
    First, by a change of variable,
    \[
        \int_0^{\min(\tau,x)} \frac {\sin(t\sqrt d)} t \diff t = \int_0^{\min(\tau,x)\cdot \sqrt d} \frac{\sin u} u \diff u\,.
    \]
    \begin{claim}[See, e.g.,~{\cite[\S 6.12]{dlmf}}]
        \label{claim:SI}
        There exists a universal
        constant $C>0$ such that for any $z> 0$,
        \[
            \left|1 - \frac 2 \pi \int_0^z \frac {\sin t} t \diff t\right| \le C\min\left(1, \frac 1 z\right)\,.
        \]
    \end{claim} 
    \noindent Applying the bound from~\Cref{claim:SI}, we get
    \[
        \int_0^\infty \left|1-\frac 2 \pi \int_{0}^{\min(\tau,x)} \frac{\sin(t\sqrt d)} t \diff t\right|\diff \varphi(x)\le O(1)\cdot \int_0^\infty \min\left(1,\frac 1 {\sqrt d \min(\tau,x)}\right) \diff \varphi(x)\,.
    \]
    We further split this integral into three parts, according to
    which of the arguments
    of the minima dominate:
    \begin{enumerate}
        \item between $0$ and $\frac 1 {\sqrt d}$, the integral
    is bounded by $O\left(\frac 1 {\sqrt d}\right)$ by the first bound.
        \item between
        $\frac 1 {\sqrt d}$ and $\tau$, $\min(\tau,x)=x$ so the integral is $O\left(\frac {\log d} {\sqrt d}\right)$ by
        the second bound.
        \item between $\tau$ and $+\infty$,
     the integral is $O\left(\frac 1 {\sqrt d}\right)$ by
    the second bound.
    \end{enumerate}
    Taken together, these observations conclude the proof.
\end{proof}

\begin{proof}[Proof of~\eqref{lem:inter-rem1}]
    Let $u(t) = \frac{e^{t^2/4} - 1} t$ extended
    by continuity with $u(0)=0$. A
    calculation shows that 
    \begin{equation}
        \max(|u(t)|,|u'(t)|)\le \frac{e^{t^2/4}} 2\quad\text{ for all $t\ge 0$.}\label{eq:decay-uup}
    \end{equation}
    
    The inner integral we want to study is
    a product of $u$ with an oscillatory factor
    of frequency $\sqrt d$, so one expects
    substantial cancellations as $d\to\infty$.
    We make this explicit by integrating by parts:
    \begin{align*}
        \int_{0}^{\min(x,\tau)} \sin(t\sqrt d) u(t) \diff t &= \left[-\frac {\cos(t\sqrt d)} {\sqrt d} u(t)\right]^{\min(x,\tau)}_{0} + \frac 1 {\sqrt d} \int_{0}^{\min(x,\tau)} \cos(t\sqrt d) u'(t)\diff t\,.
    \end{align*}
    Hence, by the triangle inequality and~\eqref{eq:decay-uup},
    \[
        \left|\int_{0}^{\min(x,\tau)} \sin(t\sqrt d) u(t) \diff t\right|\le \frac {(1+x) e^{x^2/4}} {2 \sqrt d}\,.
    \]
    Integrating the right-hand side against $\varphi$ yields $O(d^{-\frac 12})$, as
    required.
\end{proof}

\begin{proof}[Proof of~\eqref{lem:inter-rem2}]
    We decompose the inner integral according to whether
    $t\le d^{-1/2}$ or $t>d^{-1/2}$, and use~\Cref{cor:small-r} on the first part:
    \begin{align*}
        \int_0^\infty \int_{0}^{{\min(x,\tau)}} {e^{t^2/4}} \frac {|r_d(t)|} t \diff t \diff \varphi(x) &\le O\left(\frac 1 {\sqrt d}\right)\cdot \int_0^\infty {xe^{x^2/4}} \diff \varphi(x)\\ 
        &+ \int_{d^{-1/2}}^\infty \int_{d^{-1/2}}^{\min(x,\tau)} {e^{t^2/4}} \frac {|r_d(t)|} t \diff t \diff \varphi(x)\,.
    \end{align*}
    The first term is $O(d^{-1/2})$, and for the second
    we use Fubini's theorem:
    \begin{align*}
        \int_{d^{-1/2}}^\infty \int_{d^{-1/2}}^{\min(x,\tau)} {e^{t^2/4}} \frac {|r_d(t)|} t \diff t \diff \varphi(x) = \int_{d^{-1/2}}^\tau \frac {|r_d(t)|} t e^{t^2/4} \bar \Phi(t)\diff t\le \int_{d^{-1/2}}^\tau \frac {|r_d(t)|} t e^{-t^2/4} \diff t\,,
    \end{align*}
    where we used the notation $\bar\Phi(t)$ for the
    Gaussian tails as in the proof of~\Cref{lem:large-hermite}, and applied the standard bound $\bar \Phi(t)\le e^{-t^2/2}$.
    
    When $t\le \tau$ and $t\ge \frac 1 {\sqrt d}$ (i.e., $\frac t {\sqrt d}\ge \frac 1 d$) the
   upper bound from~\Cref{claim:plancherel-rotach} simplifies
    to:
    \[
        |r_d(t)|\le O\left(\frac {t^3} {\sqrt d} + \frac t {\sqrt d} + \frac {t^2} d\right)\,.
    \]
    Substituting this into the  integral above, we obtain:
    \[
        \int_{d^{-1/2}}^\tau \frac {|r_d(t)|} t e^{-t^2/4} \diff t\le O\left(\frac 1 {\sqrt d}\right)\,,
    \]
    as desired.
\end{proof}